\documentclass{article}
\pdfoutput=1

\usepackage[preprint]{neurips_2022}

\usepackage{graphicx}
\usepackage{microtype}
\usepackage{subfigure}
\usepackage{booktabs}
\usepackage{hyperref}
\usepackage{amsmath}
\usepackage{amssymb}
\usepackage{amsthm}
\usepackage{multirow}
\usepackage{enumitem}

\DeclareMathOperator{\A}{\mathrm{A}}

\DeclareMathOperator{\V}{\mathrm{V}}
\DeclareMathOperator{\E}{\mathrm{E}}

\DeclareMathOperator{\Z}{\mathrm{Z}}
\DeclareMathOperator{\X}{\mathrm{X}}
\DeclareMathOperator{\G}{\mathrm{G}}

\DeclareMathOperator{\I}{\mathrm{I}}
\DeclareMathOperator{\Y}{\mathrm{Y}}

\DeclareMathOperator{\Dsym}{\mathrm{D}^{-\frac{1}{2}}}

\newcommand\gcat[0]{GCat}
\DeclareMathOperator{\mi}{\mathrm{I}}
\DeclareMathOperator{\ent}{\mathrm{H}}
\DeclareMathOperator{\AX}{\mathrm{AX}}
\newcommand\independent{\protect\mathpalette{\protect\independenT}{\perp}}
\def\independenT#1#2{\mathrel{\rlap{$#1#2$}\mkern2mu{#1#2}}}

\theoremstyle{definition}
\newtheorem{thm}{Theorem}
\newtheorem{lem}{Lemma}

\newtheorem{defn}{Definition}

\usepackage{verbatim}
\usepackage{pgf,tikz}
\usepackage{pgfplots}
\pgfplotsset{
  ylabel style={overlay},
  yticklabel style={overlay},
}
\usetikzlibrary{backgrounds,automata}
\pgfplotsset{every tick label/.append style={font=\tiny}}

\usepackage{pgfplotstable}
\usepgfplotslibrary{dateplot}
\begin{document}
\title{Demystifying Graph Convolution \\with a Simple Concatenation}
\author{%
  Zhiqian Chen \\
  Department of Computer Science and Engineering \\
  Mississippi State University\\
  zchen@cse.msstate.edu \\
%   % examples of more authors
  \And
  Zonghan Zhang\\
  Department of Computer Science and Engineering\\
  Mississippi State University\\
  \texttt{zz239@msstate.edu}
  }
%\institute{Paper under double-blind review}
%
\maketitle              % typeset the header of the contribution
\begin{abstract}
Graph convolution (GConv) is a widely used technique that has been demonstrated to be extremely effective for graph learning applications, most notably node categorization. On the other hand, many GConv-based models do not quantify the effect of graph topology and node features on performance, and are even surpassed by some models that do not consider graph structure or node properties. We quantify the information overlap between graph topology, node features, and labels in order to determine graph convolution's representation power in the node classification task. In this work, we first determine the linear separability of graph convoluted features using analysis of variance. Mutual information is used to acquire a better understanding of the possible non-linear relationship between graph topology, node features, and labels. Our theoretical analysis demonstrates that a simple and efficient graph operation that concatenates only graph topology and node properties consistently outperforms conventional graph convolution, especially in the heterophily case. Extensive empirical research utilizing a synthetic dataset and real-world benchmarks demonstrates that graph concatenation is a simple but more flexible alternative to graph convolution.

% \keywords{First keyword  \and Second keyword \and Another keyword.}
\end{abstract}
\section{Introduction}
While Graph Neural Networks (GNNs) dominate graph learning tasks, they are generally black-box models with no explanations for their output, making their applications untrustworthy and hampering their growth into other domains~\cite{zhou2018graph,zhang2018deep,wu2019comprehensive,bronstein2017geometric}. As a result, numerous articles proposed explainable methods for determining the underlying mechanism of GNNs via post-hoc tactics~\cite{yuan2020explainability,pope2019explainability,ying2019gnnexplainer}, enabling statistical study of GNN behaviors. The bulk of post-hoc procedures are structured around the framework of sensitivity analysis~\cite{saltelli2004sensitivity} and perturbation test~\cite{kato2013perturbation} framework. Post-hoc models' fundamental drawback is their lack of theoretical basis, which might result in an instable explanation. Another related issue is to improve our understanding of GNNs through examination of their kernels and interpretation of their graph operations utilizing a variety of theoretical foundations~\cite{xu2018powerful,oono2019graph}.

The majority of existing research, on the other hand, does not quantify the effect of the input (i.e., graph topology and node attributes) on the output (i.e., node-level prediction), which continues to conceal GNNs in the learning process. Inadequate knowledge of such quantitative analysis may result in \textbf{(1)} confusion caused by irrelevant input/dimensions, wasteful and meaningless experiments, as well as \textbf{(2)}considerable computing overhead or performance reduction caused by redundant target features. To overcome this constraint, we quantify the overlap between the input and output using analysis of variance (ANOVA)~\cite{st1989analysis} and information theory~\cite{brillouin2013science}.

In this article, we will discuss the most often used GNN, graph convolutional (GConv) networks \cite{GCN}. As graph concatenation, we suggest concatenating the graph topology and node features (GCat) directly. 
Our goals are to increase the transparency of the graph convolution process by comparing GCat and GConv in homophily and heterophily circumstances, as well as to determine whether graph convolution is necessary given the input graph and node properties. 
We examine the theoretical relationship between graph convolution and graph concatenation from two perspectives: \textbf{linear dependency}, in which we use variance analysis and the Fisher score (F-test) to determine the linear separability of the input features and labels, and \textbf{non-linear dependency}, in which we quantify the information overlap between the inputs (graph topology and node features) and the output. We show that GCat not only has a higher linear separability than GConv, but also captures more information that is overlapping between the input and output in the node classification job. Additionally, because matrix concatenation is faster than matrix multiplication, GCat is more efficient than graph convolution. 
Extensive studies using synthetic data validates our inferences regarding linear and nonlinear connections in all conceivable contexts. Additionally, investigations on benchmark datasets reveal that graph concatenation outperforms state-of-the-art GNNs while being more computationally efficient.

\section{Quantifying Graph Convolution}
In this part, we define graph convolution and graph concatenation in the context of the node classification problem. We then establish their link using linear separability and mutual information.

\vspace{2pt}\noindent\textbf{Problem Setup and Preliminaries.} 
% In the graph there is no W. But A is used for the normalized adjacency matrix with self loop. Need modification here.
Consider a graph $\G=\{\V, \E, $W$\}$, where $\V$ stands for a set of $N$ nodes, $\E$ represents links among the nodes and $W=[w_{ij}]^{N\times N}$ is the adjacent matrix. $\X$ is the feature matrix of nodes, and $\Y$ is labels of them. Each node $v_{i}$ has a feature vector which is the i-th row of $\X$. 
In this paper, our goal is to study the behavior of graph convolution in the node classification, which is defined as:
\begin{equation}
\small
    \Y = f(\G, \X),
\end{equation}where $\X\in \mathbb{R}^{N\times F}$ represent raw node feature and $F$ is the feature dimension. $\Y$ denotes the class of nodes. The task is to predict labels of unlabelled nodes. Below we present the graph convolution~\cite{GCN} and a simple concatenation, called graph concatenation.
% \paragraph{Graph Convolution}
\begin{defn}[\textbf{Graph Convolution (GConv) and Graph Concatenation (GCat)}]\small
Graph convolution is an operation between the graph matrix and node feature, which is defined as
% \begin{equation*}
   $\Z_{GConv}= \AX$,
% \end{equation*}
where $\Z_{(\cdot)}$ represents the output, and $\A$ is the normalized adjacency matrix with self-loop, i.e., $\A=\Dsym(\I+W)\Dsym$, ($\textrm{D}$ is the degree matrix, $\textrm{I}$ is the identity matrix). Graph concatenation is written as
% \begin{equation*}
    $\Z_{\gcat} = \A\oplus\X$,
%\end{equation*}
where $\oplus$ means the concatenate operation between matrix $\A$ and $\X$.
\label{def:gconv_gcat}
\end{defn}

\subsection{Perspective of Linear Separability }
In this subsection, we study the linear separability of applying graph convolution and graph concatenation. First, we define homophily and heterophily for our later analysis:
\begin{defn}[\textbf{\underline{Homophily} and \underline{Heterophily} in Graph}]\small
Let $\Y$ be label value, then the 2-Dirichlet energy defined below is used as global smoothness:
{\small\begin{equation*} S_{2}(\Y) =\frac{1}{2} \sum_{(ij) \in (\V,\V)} w_{i, j}[\Y(j)-\Y(i)]^{2}.
\end{equation*}} Similarly, smoothness of feature $\X$ can be defined as {\small\begin{equation*}
S_{2}(\X) =\frac{1}{2}\sum_{k} \sum_{(i,j) \in (\V,\V)} w_{i, j}[\X_{k}(j)-\X_{k}(i)]^{2},
\end{equation*}}where $k$ is the feature index. 
When $S_{2}(\Y) \text{ or } S_{2}(\X) \approx 0$, the label or feature satisfies homophily regarding the graph. Otherwise, if $S_{2}(\Y) \text{ or } S_{2}(\X) \gg 0$, the label or feature is of heterophily in the graph. 
\label{def:lable_homo_heter}
\end{defn}
Any graph can be treated as a mix of homophily and heterophily. Therefore, we investigate linear separability of graph convolution and graph concatenation under homophily and heterophily, respectively.

\begin{thm}[\textbf{Linear Separability of Graph Convolution and Graph Concatenation}]\small
Let $\X=\{\X_{homo}, \X_{heter}\}$, denoting a complete set with two types of nodes, i.e., with the homogeneous and heterogeneous neighborhood. $J$ denotes the linear separability by Fisher score of $\X$, then we have:
\begin{equation*}\footnotesize
    J^{\gcat}(\X) \geq J^{GConv}(\X),
\end{equation*}where $J^{GCat}$ and $J^{GConv}$ are Fisher scores of graph concatenation and graph convolution, respectively.
\label{thm:1}
\end{thm}

\textit{For details of proof for Theorem \ref{thm:1}, see supplementary material.} 
Theorem \ref{thm:1}'s conclusion is based on the fact that GConv smooths the neighborhood, hence obliterating difference information, whereas GCat is adaptive to the neighborhood situation and more adaptable to unseen cases.

\subsection{Perspective of Mutual Information}
This subsection discusses the behavior of graph convolution and graph concatenation in terms of mutual information.
% outlining their differences in all possible cases.
First, we define the way to calculate mutual information of graph convolution and graph concatenation. Then the relationship between their mutual information are derived by our case by case study.
\begin{figure*}
    \centering
    \includegraphics[width=1\textwidth]{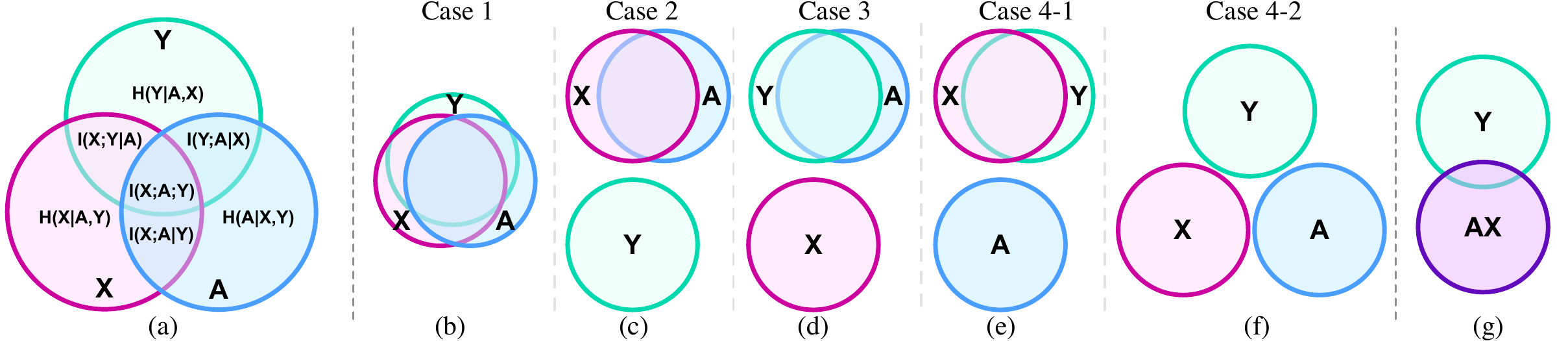}
    \caption{(a) shows a Venn diagram of information theoretic measures for three variables $\X$,$\A$, and $\Y$, represented by the lower left, lower right, and upper circles, respectively. Green/red/blue circles denote entropy of $\Y$, $\X$ and $\A$ respectively (i.e., $\ent(\Y), \ent(\X), \ent(\A)$). (b)-(f) discuss all possible cases among them, while (g) illustrates the information overlap under graph convolution.}
    \label{fig:mi_thm}
\end{figure*}
\begin{defn}[\textbf{Mutual Information of Graph Convolution and Graph Concatenation}]
Mutual information of graph convolution is defined as
% \begin{equation*}
   $\mi_{GConv} = \mi(\AX;\Y)$,
% \end{equation*}
which means the mutual information of graph convolution, showing how much useful information that graph convolution ($\AX$) can capture $\Y$.
Mutual information of graph concatenation is defined as
% \begin{equation*}
    $\mi_{\gcat} = \mi(\X, \A; \Y)$,
% \end{equation*}
which denotes the mutual information of graph concatenation, measuring the overlap information between the union of $(\X, \A)$ and $\Y$. Note $\A$ alone in the equations above can be represented by node embeddings that only utilize graph structure such as random walk \cite{perozzi2014deepwalk}.
\label{def:mi_gc_gcat}
\end{defn}

\begin{lem}[\textbf{Mutual Information of \textbf{\underline{Homophily} and \underline{Heterophily}}}]\small
If label $\Y$ satisfies homophily or heterophily in a graph, then $\Y$ is dependent or independent from the adjacency matrix $\A$. So the mutual information between $\Y$ and $\A$ can be written as:  
\begin{equation*}
\mi(\Y;\A)\left\{
 \begin{array}{ll}
\gg0, & \text{Y is homophily }\\
= 0, & \text{Y is heterophily}
\end{array}\right.
\end{equation*}It is easy to get:
\begin{equation*}
\ent(\Y|\A) \text{  or  } \ent(\A|\Y) \left\{
\begin{array}{ll}
=0, & \text{Y is homophily }\\
\gg0, & \text{Y is heterophily}
\end{array}\right.
\end{equation*}where $\ent$ denotes the entropy, and $\ent(\Y|\A)$ represents the conditional entropy which means the information in $\Y$ but not in $\A$ (see \cite{brillouin2013science} for the relationship between (conditional) entropy and mutual information; see all possible relationships among $\X$, $\Y$ and $\A$ in Fig. \ref{fig:syn_data}a).
\label{lem:mi_homo_heter}
\end{lem}

For feature $\X$, we have exactly the same conclusion of Lemma \ref{lem:mi_homo_heter}, only by replacing $\A$ with $\X$. Thus, we have

\begin{thm}[\textbf{Relationship between the Mutual Information of Graph Convolution and Graph Concatenation}]\small
Given Def. \ref{def:mi_gc_gcat}, we have:
\begin{equation*}
    \mi_{\gcat} \geq \mi_{GConv}
\end{equation*}
\label{thm:2}
\end{thm}
\begin{proof}\small
Utilizing the relationship between mutual information and conditional entropy~\cite{doi:10.1162/089976603321780272}, we have:
\begin{align*}
\mi_{\Delta} =& \mi_{GCat} - \mi_{GConv} &\\
   =&[\ent(\Y)-\ent(\Y|\AX)] - [\ent(\Y)-\ent(\Y|\A,\X)] &\\
   =&\ent(\Y|\AX)-\ent(\Y|\A,\X),
\end{align*}where $\ent(\Y|\A,\X)$ is the information of $\Y$ excluding the overlap area with $\A$ or $\X$ (see Fig. \ref{fig:mi_thm}a). Therefore, we will prove $\mi_{\Delta}>0$
\vspace{2pt}\noindent\underline{\textbf{Case 1}}: $\X$ is homophily and $\Y$ is homophily in $\G$ (Fig. \ref{fig:mi_thm}b). $\AX$ is also homophily regarding $\Y$, so we have:
\begin{equation*}\small
    \mi_{\Delta} = \ent(\Y|\AX)-\ent(\Y|\A,\X)= 0 - 0 = 0
\end{equation*}
\vspace{2pt}\noindent\underline{\textbf{Case 2}}: $\X$ is homophily and $\Y$ is heterophily in $\G$ (Fig. \ref{fig:mi_thm}c). $\A$, $\X$, and $\AX$ are all independent from $\Y$, so we have:
\begin{equation*}\small
    \mi_{\Delta} = \ent(\Y|\AX)-\ent(\Y|\A,\X)= \ent(\Y) - \ent(\Y) = 0
\end{equation*}
\vspace{2pt}\noindent\underline{\textbf{Case 3}}: $\X$ is heterophily and $\Y$ is homophily regarding the graph $\G$ (Fig. \ref{fig:mi_thm}d). Therefore, $\A$ has large overlap with $\Y$. Graph convolution ($\AX$ as shown in Fig. \ref{fig:mi_thm}g) smooths $\X$, but cannot make it completely homophily unless applying graph convolution many times (i.e., over-smoothing). Therefore, $\AX$ still has partial overlap with with $\Y$, that can be written as:
\begin{equation}\small
    \ent(\Y|\AX) = \ent(\Y)  - \mi(\Y, \AX) >0.
    \label{eq:case3}
\end{equation}
Since $\Y$ is homophily regarding the graph.
\begin{align*}\small
    \mi_{\Delta} =& \ent(\Y|\AX)-\ent(\Y|\A,\X)&\\
    =& \ent(\Y|\AX)-\ent(\Y|\A) & (\text{\footnotesize$\X \independent \Y$})\\
    >&0 +0 = 0& (\text{Eqn.\ref{eq:case3}, Lem. \ref{lem:mi_homo_heter}})
\end{align*}where $\independent$ means independence (due to that $\Y$ is heterophily and $\X$ is homophily).
\vspace{2pt}\noindent\underline{\textbf{Case 4}}: $\X$ is heterophily and $\Y$ is heterophily in $\G$. There are two subcases in terms of mutual information: (4-1) $\X$ has a overlap with $\Y$ (Fig. \ref{fig:mi_thm}e), and (4-2) $\X$ has no overlap with $\Y$ (Fig. \ref{fig:mi_thm}f). Case 4-1 is similar to case 3 by switching $\X$ and $\A$. So we have:
\begin{align*}\small
    \mi_{\Delta} =& \ent(\Y|\AX)-\ent(\Y|\A,\X)&\\
    =& \ent(\Y|\AX)-\ent(\Y|\X) & (\text{\footnotesize$\A \independent \Y$})\\
    >&0 +0=0& (\text{Eqn.\ref{eq:case3}, Lem. \ref{lem:mi_homo_heter}})
\end{align*}
As for case 4-2 (Fig. \ref{fig:mi_thm}f), none of $\X, \A, \Y$ is overlapped with one another. Therefore, $\mi_{\Delta}=0$.
Summarizing case 1-4, we conclude that $\mi_{\gcat} \geq \mi_{GConv}$.
\end{proof}
% \noindent\textbf{Time Complexity Analysis.}
\vspace{2pt}\noindent\textbf{Time Complexity Analysis.}
Graph convolution~\cite{GCN} is the product of $\A$ and $\X$, as defined in Def. \ref{def:gconv_gcat}. Therefore, its time complexity is $\mathcal{O}(N^2F)$. For graph concatenation, the complexity is $\mathcal{O}(1)$ since it is a matrix concatenation. Therefore, there is a remarkable superiority of graph concatenation beyond graph convolution.

\section{Related Works}
% In this section, we discuss two lines of related work that dissect graph neural networks, and the relationship between our work and existing works.

\textbf{Explainable GNNs}
As an extension of deep learning, graph neural networks are also confronted with the black box issue. There are several threads of applying the explainable approach to demystify the learning process of GNNs. One topic is to learn input importance from the gradient~\cite{baldassarre2019explainability,pope2019explainability}. Another popular perspective is from sensitivity analysis, which perturbs minor component and evaluate its influence on the global level \cite{ying2019gnnexplainer,luo2020parameterized}. Similarly, ~\cite{huang2020graphlime,vu2020pgm} provide instance-level explanations, while \cite{yuan2020xgnn} offers model-level interpretation.
Different from existing works, our work seeks to quantify, in the input and output level, how much the input affects the output with linear and non-linear dependency analysis.

\noindent\textbf{Over-smoothing Issue of GNNs}
Similar to normal deep neural networks, graph convolution over-smooth the features as the number of layers increases under homophily case. To solve this issue, many recent works are proposed to mitigate the performance decrease~\cite{zhu2020beyond, bo2021beyond, li2019deepgcns, zhao2019pairnorm, oono2019graph,rong2019dropedge,huang2020tackling,chen2020simple,liu2020towards,zhou2020deeper,yang2020revisiting}. 
However, a case has been totally ignored: when graph features or labels are heterophily regarding the graph structure, which is added in the discussion throughout this paper.

\section{Experiments}
Data and code for synthetic and real-world experiments are provided \footnote{https://bitbucket.org/gconcat2021/graph-concatenation/downloads/}.
\subsection{Synthetic Experiment}
In the synthetic data, Minnesota road network is used as a graph structure for visualization. Based on this graph, $\X$ and $\Y$ are generated with homophily or heterophily. Then, the calculation of Fisher score and mutual information are performed to derive dependency relationship among the input and output. Finally, we conducted a node classification to study its connection with dependency analysis.

\vspace{2pt}\noindent\textbf{Data and Metrics.}
Specifically, cases 1, 2, 3, and 4-1 will be studied in this section with binary classification (case 4-2 is meaningless since no data is related to one another).
There are two types of data regarding the graph: homophily and heterophily, which are composed respectively: \textbf{[Homophily]}: low-frequency eigenvectors of graph Laplacian are used, and ten different eigenvectors are selected at the index of \{1, 4, 7, 10, 13, 16, 19, 22, 25, 28\}. 
\textbf{[Heterophily]}: Similarly, eigenvectors at index of \{-1, -4, -7, -10, -13, -16, -19, -22, -25, -28\} are chosen for heterophily, since they are high-frequency components. Another option is to randomly set a group of scattering points as one class, leaving the remaining nodes as the other class. Different point numbers are tested: \{10, 50, 100, 200, 300, 500, 1000, 1500, 2000, 2500\}. 
Samples for synthetic data are illustrated in Fig. \ref{fig:syn_data} in supplementary material. Note that the graph ($\A$) is fixed, so we only need to configure $\X$ and $\Y$. Detailed configuration for each case is listed below. The representaion of $\A$ in GCat is obtained by DeepWalk~\cite{perozzi2014deepwalk} with 256 dimensions.
% \begin{itemize}[\leftmargin=*]
% \small
\noindent\underline{\textbf{Case 1}}: $\X$ is set to be homophily w.r.t. $\A$, and is set to be low-frequency eigenvector. $\Y$ is set to be discretized $\X$, i.e., nodes with values that are higher than the median of $\X$ are set to one class, and those that are lower than median are set to the other class.
\noindent\underline{\textbf{Case 2}}: $\X$ is set to be homophily, and  $\Y$ is set to be heterophily with two options describe above: \texttt{(option 1)} $\Y$ with different eigenvectors for heterophily or \texttt{(option 2)} $\Y$ with different point numbers for heterophily;  \noindent\underline{\textbf{Case 3}}: $\Y$ is set to be homophily, and  $\X$ is set to be heterophily with the same two ways described above. 
\noindent\underline{\textbf{Case 4-1}}: $\X$ is set to be heterophily, and $\Y$ is set to be discretized $\X$ as shown in case 1. There are also \texttt{(option 1)} and \texttt{(option 2)} for setting $\X$.
% \end{itemize}
% \subsubsection{Metrics}
We compare the Fisher score and mutual information of graph convolution and graph concatenation, respectively. Specifically, $J(\AX)$ (w.r.t. $\Y$) and $\I(\AX;\Y)$ are evaluated for graph convolution, and $J(\A\oplus\X)$ (w.r.t. $\Y$) and $\I(\A\oplus\X;\Y)$ are for graph concatenation. Higher value of both Fisher score and mutual information indicates better linear separability and more information overlap, and thereby a better transform.
% \usepackage{booktabs}
% \usepackage{booktabs}

% \begin{tikzpicture}
% 	\begin{axis}[stack plots=y]
% 	\addplot coordinates
% 		{(0,1) (1,1) (2,2) (3,2)};
% 	\addplot coordinates
% 		{(0,1) (1,1) (2,2) (3,2)};
% 	\addplot coordinates
% 		{(0,1) (1,1) (2,2) (3,2)};
% 	\end{axis}
% \end{tikzpicture}

\begin{figure}[!hbpt]
\centering
\subfigure 
{ 
\begin{tikzpicture}
\begin{axis}[
width=0.28\linewidth,
height=0.32\linewidth,
zmin=0, zmax=1.8e-3,
% xlabel = X,
% ylabel = Y,
% zlabel = $z$,
% title = {A Scatter Plot Example}
% label style={font=\scriptsize},
ticklabel style = {font=\tiny},
colormap/cool,
%   colorbar,
legend style={nodes={scale=0.65, transform shape},fill=white, fill opacity=0.7, draw opacity=1,text opacity=1}
]
% this yields a 3x4 matrix in colum
\addplot3 [surf,mesh/check=false, mesh/rows=30,mesh/ordering=y varies,opacity=0.85] file {data/case_2_0_gcat_fs_max.dat};
\addlegendentry{GCat $J$};
% \addplot3 [only marks]  file {data_3d_grid.dat};
\end{axis}
\end{tikzpicture}
}
\subfigure{
\begin{tikzpicture}
\begin{axis}[
width=0.28\linewidth,
height=0.32\linewidth,
zmin=0, zmax=1.8e-3,
% xlabel = $x$,
% ylabel = $y$,
% zlabel = {$f(x,y) = x \cdot y$},
% title = {A Scatter Plot Example}
colormap/cool,
%   colorbar,
legend style={nodes={scale=0.65, transform shape},fill=white, fill opacity=0.7, draw opacity=1,text opacity=1}
]
% this yields a 3x4 matrix in colum
\addplot3 [surf,mesh/check=false,mesh/rows=30,mesh/ordering=y varies,opacity=0.85] file {data/case_2_0_gconv_fs.dat};
\addlegendentry{GConv $J$};
% \addplot3 [only marks]  file {data_3d_grid.dat};
\end{axis}
\end{tikzpicture}}
\subfigure 
{ 
\begin{tikzpicture}
\begin{axis}[
width=0.28\linewidth,
height=0.32\linewidth,
zmin=0, zmax=3.5e-2,
% xlabel = $x$,
% ylabel = $y$,
% zlabel = {$f(x,y) = x \cdot y$},
% title = {A Scatter Plot Example}
colormap/redyellow,
%   colorbar,
legend style={nodes={scale=0.65, transform shape},fill=white, fill opacity=0.7, draw opacity=1,text opacity=1}
]
% this yields a 3x4 matrix in colum
\addplot3 [surf,mesh/check=false,mesh/rows=30,mesh/ordering=y varies,opacity=0.85] file {data/case_2_0_gcat_mi_max.dat};
\addlegendentry{GCat $\I$};
% \addplot3 [only marks]  file {data_3d_grid.dat};
\end{axis}
\end{tikzpicture}
}
\subfigure{
\begin{tikzpicture}
\begin{axis}[
width=0.28\linewidth,
height=0.32\linewidth,
zmin=0, zmax=3.5e-2,
% xlabel = $x$,
% ylabel = $y$,
% zlabel = {$f(x,y) = x \cdot y$},
% title = {A Scatter Plot Example}
colormap/redyellow,
%   colorbar,
legend style={nodes={scale=0.65, transform shape},fill=white, fill opacity=0.7, draw opacity=1,text opacity=1}
]
% this yields a 3x4 matrix in colum
\addplot3 [surf,mesh/check=false,mesh/rows=30,mesh/ordering=y varies,opacity=0.85] file {data/case_2_0_gconv_mi.dat};
\addlegendentry{GCon $\I$};
% \addplot3 [only marks]  file {data_3d_grid.dat};
\end{axis}
\end{tikzpicture}
}
\subfigure 
{ 
\begin{tikzpicture}
\begin{axis}[
width=0.28\linewidth,
height=0.32\linewidth,
zmin=0, zmax=6e-3,
ytick={500,1500},
% xlabel = $x$,
% ylabel = $y$,
% zlabel = {$f(x,y) = x \cdot y$},
% title = {A Scatter Plot Example}
colormap/cool,
%   colorbar,
legend style={nodes={scale=0.65, transform shape},fill=white, fill opacity=0.7, draw opacity=1,text opacity=1}
]
% this yields a 3x4 matrix in colum
\addplot3 [surf,mesh/check=false,mesh/rows=10,mesh/ordering=y varies, opacity=0.85] file {data/case_2_1_gcat_fs_max.dat};
\addlegendentry{GCat $\I$};
% \addplot3 [only marks]  file {data_3d_grid.dat};
\end{axis}
\end{tikzpicture}}
\subfigure{
\begin{tikzpicture}
\begin{axis}[
width=0.28\linewidth,
height=0.32\linewidth,
zmin=0, zmax=6e-3,
ytick={500,1500},
% xlabel = $x$,
% ylabel = $y$,
% zlabel = {$f(x,y) = x \cdot y$},
% title = {A Scatter Plot Example}
colormap/cool,
%   colorbar,
legend style={nodes={scale=0.65, transform shape},fill=white, fill opacity=0.7, draw opacity=1,text opacity=1}
]
% this yields a 3x4 matrix in colum
\addplot3 [surf,mesh/check=false,mesh/rows=10,mesh/ordering=y varies, opacity=0.85] file {data/case_2_1_gconv_fs.dat};
\addlegendentry{GConv $\I$};
% \addplot3 [only marks]  file {data_3d_grid.dat};
\end{axis}
\end{tikzpicture}}
\subfigure{
\begin{tikzpicture}
\begin{axis}[
width=0.28\linewidth,
height=0.32\linewidth,
zmin=0, zmax=3e-2,
ytick={500,1500},
% xlabel = $x$,
% ylabel = $y$,
% zlabel = {$f(x,y) = x \cdot y$},
% title = {A Scatter Plot Example}
colormap/redyellow,
%   colorbar,
legend style={nodes={scale=0.65, transform shape},fill=white, fill opacity=0.7, draw opacity=1,text opacity=1}
]
% this yields a 3x4 matrix in colum
\addplot3 [surf,mesh/check=false,mesh/rows=10,mesh/ordering=y varies, opacity=0.85] file {data/case_2_1_gcat_mi_max.dat};
\addlegendentry{GCat $J$};
% \addplot3 [only marks]  file {data_3d_grid.dat};
\end{axis}
\end{tikzpicture}
}
\subfigure{
\begin{tikzpicture}
\begin{axis}[
width=0.28\linewidth,
height=0.32\linewidth,
zmin=0, zmax=3e-2,
ytick={500,1500},
xlabel = X,
ylabel = Y,
zlabel = Z,
label style={font=\tiny},
colormap/redyellow,
%   colorbar,
legend style={nodes={scale=0.65, transform shape},fill=white, fill opacity=0.7, draw opacity=1,text opacity=1},
% label style={font=\tiny}
]
% this yields a 3x4 matrix in colum
\addplot3 [surf,mesh/check=false,mesh/rows=10,mesh/ordering=y varies, opacity=0.85] file {data/case_2_1_gconv_mi.dat};
\addlegendentry{GConv $J$};
% \addplot3 [only marks]  file {data_3d_grid.dat};
\end{axis}
\end{tikzpicture}
}
\caption{Case 2: \textbf{x-axis}: $\X$. \textbf{y-axis}: \textbf{[Top Four (option 1)]} $\Y$ is set to be eigenvectors (y-axis), or \textbf{[Bottom Four (option 2)]} $\Y$ is set to point numbers. \textbf{z-axis} indicates the values of $J$ or $\I$. (all axis are marked in the last sugfigure)}
\label{fig:case_2}
\end{figure}
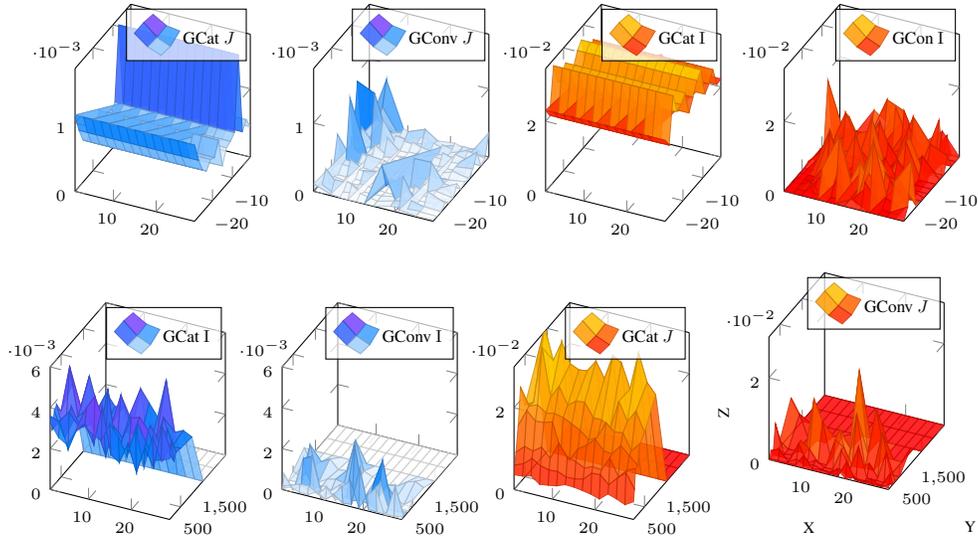
\vspace{2pt}\noindent\textbf{Results and Analysis.}
Results of \textbf{\underline{Case 1}} are shown in Fig. \ref{fig:case_1_mi} of supplementary material: GCat and GConv have almost the same Fisher scores, and their mutual information are at the same level. Note that there is only 2\% difference in the mutual information.
\textbf{\underline{Case 2}} has two options as presented in data generation. In each option, the Fisher score and mutual information are evaluated, and there are eight single evaluations in total.
% i.e., GCat's Fisher score in \texttt{option 1}, GCat's mutual information in \texttt{option 1} and so on. 
There are 100 combinations between $\X$ and $\Y$. Fig. \ref{fig:case_2} shows a value surface by interpolating the 100 combinations. The top four subfigures are \texttt{option 1}, and subfigures with blue schema are the Fisher scores of GCat and GConv. It's clear that the surface of GCat is dramatically higher than that of GConv. Those subfigures with red/orange are for the mutual information, and GCat's is significantly higher than GConv's. With \texttt{option 2} (the bottom four subfigures), GCat's still has similar superiority. 
Supplementary material's Fig. \ref{fig:case_3} depicts the \textbf{\underline{Case 3}} with the same style, most GCat's surfaces are much higher except for in the third and fourth subfigures in the top row (still higher by a small margin). 
% All detailed numbers in Fig. \ref{fig:case_2} and \ref{fig:case_3} are listed in appendix.
In \textbf{\underline{Case 4-1}} (Fig. \ref{fig:case_4_mi}), GCat shows slight advantage beyond GConv under the Fisher score (left column) and mutual information (right column).

\begin{figure}
\centering
\subfigure{
\begin{tikzpicture}
\begin{axis}[
legend pos=south east,
width=0.58\linewidth,
height=0.2\linewidth,
legend style={nodes={scale=0.65, transform shape},fill=white, fill opacity=0.7, draw opacity=1,text opacity=1}
]
\addplot table {data/case_4_0_gcat_fs_max.dat};
\addplot table {data/case_4_0_gconv_fs.dat};
\legend{gcat, gconv}
\end{axis}
\end{tikzpicture}}
\hfill
\subfigure{
\begin{tikzpicture}
\begin{axis}[
legend pos=south east,
width=0.58\linewidth,
height=0.2\linewidth,
legend style={nodes={scale=0.65, transform shape},fill=white, fill opacity=0.7, draw opacity=1,text opacity=1}
]
\addplot table {data/case_4_0_gcat_mi_max.dat};
\addplot table {data/case_4_0_gconv_mi.dat};
\legend{gcat, gconv}
\end{axis}
\end{tikzpicture}}
\subfigure{
\begin{tikzpicture}
\begin{axis}[
legend pos=south east,
width=0.58\linewidth,
height=0.2\linewidth,
legend style={nodes={scale=0.65, transform shape},fill=white, fill opacity=0.7, draw opacity=1,text opacity=1}
]
\addplot table {data/case_4_1_gcat_fs_max.dat};
\addplot table {data/case_4_1_gconv_fs.dat};
\legend{gcat, gconv}
\end{axis}
\end{tikzpicture}}
\hfill
\subfigure{
\begin{tikzpicture}
\begin{axis}[
legend pos=south east,
width=0.58\linewidth,
height=0.2\linewidth,
legend style={nodes={scale=0.65, transform shape},fill=white, fill opacity=0.7, draw opacity=1,text opacity=1}
]
\addplot table {data/case_4_1_gcat_mi_max.dat};
\addplot table {data/case_4_1_gconv_mi.dat};
\legend{gcat, gconv}
\end{axis}
\end{tikzpicture}}
\caption{Case 4: \textbf{(Left Column)} Fisher scores of graph concatenation and graph convolution. x-axis is eigenvector index. \textbf{(Right Column)} Mutual information of graph concatenation and graph convolution, x-axis is point number. Top row is for \texttt{option 1}, and bottom row for \texttt{option 2}.}
\label{fig:case_4_mi}
\end{figure}
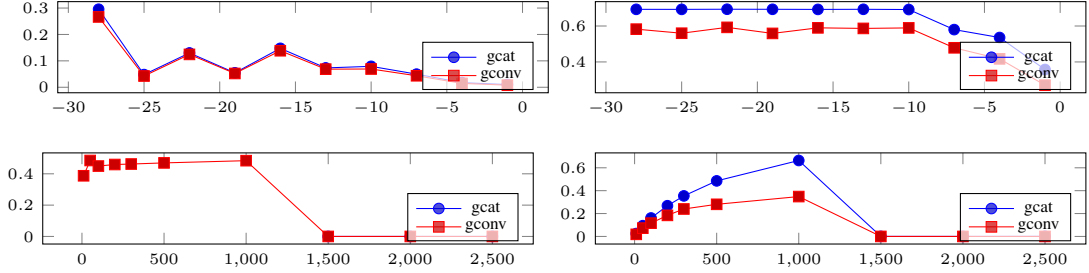

To study the relationship between Fisher score/mutual information and classification accuracy, we conducted experiments with several models based on graph convolution and graph concatenation. Specifically, 
{
\small
\begin{itemize}
    \item \textbf{GCat + LR}: Using the output features of graph concatenation, logistic regression or SVM is applied to predict the testing node.
        \item \textbf{GCat + SVM}: Using the output features of graph concatenation, SVM is applied to predict the testing node.
    \item \textbf{GConv + LR}: Using the output features of graph convolution, logistic regression or SVM is applied to predict the testing node.
        \item \textbf{GConv + SVM}: Using the output features of graph convolution, SVM is applied to predict the testing node.
            \item \textbf{GCN}: Directly apply graph convolutional networks~\cite{GCN}
\end{itemize}}

The train/test split is 60\%/40\% with shuffling, and the experiments are repeated 10 times. Fig. \ref{fig:syn_classif} shows the results: 
\textbf{\underline{Case 1}} is easy task since all variables ($\X$, $\Y$, $\A$) are highly correlated. GCat and GConv performed very good (around 90\%). 
\textbf{\underline{Case 2}} is difficult due to that $\Y$ has no overlap with $\X$ or $\A$. GCat and GConv both have lower accuracy (around 85\%) for \texttt{option 2}, and under 50\% for \texttt{option 1}. 
In  \textbf{\underline{Case 3}} (\texttt{option 1} and \texttt{option 2}), GConv shows disadvantage since graph convolution dilutes useful information of $\A$ by averaging $\A$ and $\X$, while GCat  keeps all the information and provides flexibility to the model. 
There is no big difference between GCat and GConv in \textbf{\underline{Case 4}}. This is may due to dimension number of $\A$ (256) is much larger than that of $\X$ (1) in GCat, so $\X$ was deweigthed in the learning.

\begin{figure*}
    \centering
% \begin{tikzpicture}
% 	\begin{axis}[
% 	width=1\linewidth,
%     height=0.5\linewidth,
% 	xtick={0,1, 2, 3, 4, 5,6},
% 	xticklabels={
%     Case 1,
%     Case 2/Opt 1,
%     Case 2/Opt 2,    
%     Case 3/Opt 1,
%     Case 3/Opt 2,    
%     Case 4/Opt 1,
%     Case 4/Opt 2,    
%     },
%     x tick label style={rotate=45,anchor=east}
% 	]
% 	\addplot coordinates
% 		{(0,0.892) (1,0.405) (2,0.838) (3,0.889) (4,0.888) (5,0.428) (6,0.999)};
% 	\addplot coordinates
% 		{(0,0.898) (1,0.405) (2,0.839) (3,0.893) (4,0.893) (5,0.430) (6,1.000)};
% 	\addplot coordinates
% 		{(0,0.947) (1,0.488) (2,0.846) (3,0.49) (4,0.501) (5,0.564) (6,0.903)};
% 	\addplot coordinates
% 		{(0,0.899) (1,0.489) (2,0.846) (3,0.49) (4,0.496) (5,0.492) (6,0.896)};
% 	\addplot coordinates
% 		{(0,0.488) (1,0.49) (2,0.837) (3,0.488) (4,0.502) (5,0.49) (6,0.942)};
% 	\end{axis}
% \end{tikzpicture}
    \label{fig:my_label}
    \begin{tikzpicture}
\begin{axis}[
    width=1\linewidth,
    height=0.34\linewidth,
    ybar,
    bar width=6pt,
    enlargelimits=0.2,
    legend style={at={(0.5,-0.3)},
      anchor=north,legend columns=-1},
    ylabel={Accuracy},
    legend style={nodes={scale=0.6, transform shape}}, 
    symbolic x coords={
        Case 1,
    Case 2/Opt 1,
    Case 2/Opt 2,    
    Case 3/Opt 1,
    Case 3/Opt 2,    
    Case 4/Opt 1,
    Case 4/Opt 2,    
    },
    xtick=data,
    nodes near coords,
    nodes near coords align={vertical},
    % hide y axis,
    every node near coord/.append style={rotate=90,anchor=west,font=\tiny}
    ]
	\addplot coordinates
		{(Case 1,0.892) (Case 2/Opt 1,0.405) (Case 2/Opt 2,0.838) (Case 3/Opt 1,0.889) (Case 3/Opt 2,0.888) (Case 4/Opt 1,0.428) (Case 4/Opt 2,0.999)};
	\addplot coordinates
		{(Case 1,0.898) (Case 2/Opt 1,0.405) (Case 2/Opt 2,0.839) (Case 3/Opt 1,0.893) (Case 3/Opt 2,0.893) (Case 4/Opt 1,0.430) (Case 4/Opt 2,1.000)};
	\addplot coordinates
		{(Case 1,0.947) (Case 2/Opt 1,0.488) (Case 2/Opt 2,0.846) (Case 3/Opt 1,0.49) (Case 3/Opt 2,0.501) (Case 4/Opt 1,0.564) (Case 4/Opt 2,0.903)};
	\addplot coordinates
		{(Case 1,0.899) (Case 2/Opt 1,0.489) (Case 2/Opt 2,0.846) (Case 3/Opt 1,0.49) (Case 3/Opt 2,0.496) (Case 4/Opt 1,0.492) (Case 4/Opt 2,0.896)};
	\addplot coordinates
		{(Case 1,0.488) (Case 2/Opt 1,0.49) (Case 2/Opt 2,0.837) (Case 3/Opt 1,0.488) (Case 3/Opt 2,0.502) (Case 4/Opt 1,0.49) (Case 4/Opt 2,0.942)};
\legend{
GCat + LR , GCat + SVM   , GCon + LR  , GConv + SVM , GCN                      
    }
\end{axis}
\end{tikzpicture}
\caption{Accuracy of node classification with different methods}
\label{fig:syn_classif}
\end{figure*}
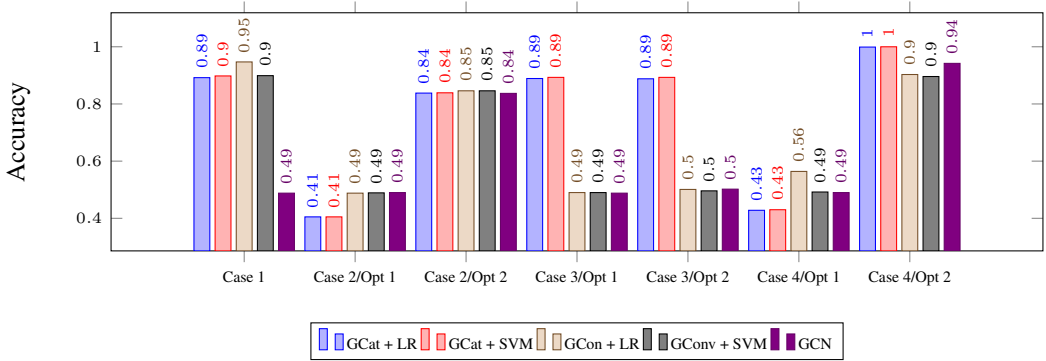

For the efficiency evaluation, graph convolution and graph concatenation are performed 100 times. The mean runtime is $0.0026\pm 0.0001$ seconds for graph convolution, and
$0.0003\pm 0.0001$ seconds for graph concatenation. So there is graph concatenation is faster than graph convolution by approximately 7.6 times.

\subsection{Real-world Experiment}
Besides the synthetic data, real-world datasets are utilized to evaluate the graph concatenation method on both prediction accuracy and training time.

\vspace{2pt}\noindent\textbf{Datasets.}
As in \cite{yang2016revisiting}, we use three citation network datasets: Cora, CiteSeer, and PubMed for evaluation. In these datasets, nodes represent documents and edges represent citation links. Each node has a corresponding feature vector capturing the bag-of-words of the document and a class label indicating its related field. 
 
\vspace{2pt}\noindent\textbf{Baselines.}
We include several GNN-based methods as baselines. These methods differ from graph concatenation in that they all utilize graph convolution or feature propagation as part of their model training process.
{\small
\begin{itemize}
    \item \textbf{Graph Convolutional Network (GCN)} is a first-order approximation of spectral graph convolutions proposed by \cite{GCN}.

%    \item \textbf{Graph Attention Network (GAT)} \cite{velivckovic2018graph} is an attention mechanism that enables specifying different weights to different nodes in a neighborhood.
    \item \textbf{Graph Isomorphism Network (GIN)} \cite{xu2018powerful} propose a GNN as powerful as  the Weisfeiler-Lehman graph isomorphism test under the neighborhood aggregation framework.

    \item \textbf{Graph Sampling and Aggregation (GraphSAGE)} \cite{hamilton2017inductive} generates embeddings by neighborhood sampling and aggregation and realizes an implementation of GCN. 

    \item \textbf{Simple Graph Convolution (SGC)} \cite{wu2019simplifying} removes the nonlinearities between GCN layers and eventually simplifies the function into a single linear transformation.

    \item \textbf{Convolutional Neural Networks on Graphs (ChebNet)} \cite{1606.09375} achieves a higher-order graph convolution operation by modeling neighborhoods with combined Chebyshev polynomials. 
\end{itemize}
}
\vspace{2pt}\noindent\textbf{Experimental Setup.}
In our experiment, we assume that we have full knowledge of features of all nodes as well as information of the entire graph. Part of the nodes with known labels are used to train the models while the other part is treated as if their labels are unknown for model testing. All baseline models adopt similar settings such that they run 200 epochs using Adam optimizer \cite{kingma2014adam} to reach the finalized parameters. During the training process, the dropout technique \cite{srivastava2014dropout} is adopted to prevent overfitting. Each experiment, on either one of the baselines or graph concatenation, is repeated ten times with the same settings. In each time, the datasets are divided into train set and test set as 60\%/40\% with shuffling. Means of prediction accuracy and training time, as well as their standard deviations, are reported as the measurements of each model's effectiveness and efficiency. 
For graph concatenation model, we adopt Deepwalk \cite{perozzi2014deepwalk} to generate graph embeddings which are later concatenated with the node features. The data matrix after concatenation is fitted into a Logistic Regression model with L2 regularization. When considering the training time cost, we only report the time spent on matrix concatenation and the logistic regression, which is performed on the concatenated matrix, based on the assumption that the graph embedding is available for the node classification task. This assumption holds under two general circumstances. The first one is that when a series of recurring prediction tasks are carried out on a static graph structure with changing node features, graph embedding can be reused after being initially learnt. The second circumstance is that even when a few new nodes are added into a graph whose embedding is known, the time cost of updating the graph embedding based on previous knowledge is way less than the time cost of the initial learning. 

\vspace{2pt}\noindent\textbf{Results.}
The results of node classification accuracy are summarized in Table ~\ref{tab:2}. Comparing to the five baselines, which are the common graph neural network methods for node classification, graph concatenation model presents a competitive predicting power. On every dataset of the three, it has an above-average prediction accuracy and exceeds the performance of more than half of the baselines. 
We further consider the performances of the models averaging on all three datasets by calculating the means of ranks each method achieves on the datasets. Results show that the graph concatenation model shares the best average rank with vanilla GCN, indicating that both of them have consistent performance across the datasets. Note that these two methods not only possess the same average rank but also have a similar prediction accuracy on each dataset. This finding suggests that graph convolution and graph concatenation provide similar proportions of mutual information regarding the node label.
\begin{table}[!htpb]
\centering
\scalebox{0.9}{
\begin{tabular}{l|rrr|c}
\hline
Method                  & Cora                  & CiteSeer              & PubMed                & Avg. Rank \\\hline
GCN                     & $86.65\pm 0.82$          & $73.12\pm 1.05$          & $87.22\pm 0.35$          & $2.33$         \\
ChebNet                 & $84.51\pm 0.95$          & $71.62\pm 1.03$          & $88.20\pm 0.27$          & $3.67$         \\
GIN                     & $82.47\pm 3.67$          & $70.41\pm 1.03$          & $84.73\pm 0.65$          & $5.33$         \\
SGC                     & $85.71\pm 0.84$          & $75.70\pm 0.80$          & $65.54\pm 1.24$          & $3.67$         \\
GraphSAGE               & $85.75\pm 0.91$          & $73.42\pm 0.66$          & $84.66\pm 0.38$          & $3.67$         \\\hline
\textbf{GCat}        & $85.89\pm 0.74$          & $73.72\pm 0.83$          & $86.60\pm 0.34$          & \textbf{$2.33$}   \\\hline     
\end{tabular}}
\caption{\label{tab:2}Mean and standard deviation of prediction accuracy for all methods on all datasets.}
\end{table}
The training time for each model is presented in Table ~\ref{tab:3}. We can see that graph concatenation is significantly faster than all the baseline methods except for SGC, which is actually not a graph convolution method. As clarified in the Experimental Setup section above, the training time reported here for graph concatenation refers to the time for concatenating two matrices and training the Logistic Regression only. Table ~\ref{tab:3} also suggests that the gap between the training times of graph concatenation and graph convolution is not as great as expected according to time complexity analysis and the synthetic experiment. This is because the training time of graph concatenation reported here is mostly taken by the process of fitting the logistic regression model. Concatenating the matrices only takes a very small proportion of the reported time, 2.35\% on the PubMed dataset for instance.

\begin{table}[!htpb]
\centering
\scalebox{0.9}{
\begin{tabular}{l|rrr}
\hline
Method                  & Cora            & CiteSeer       & PubMed          \\\hline
GCN                     & $496.16 \pm 1.93$   & $499.97 \pm 3.42$  & $916.17 \pm 0.90$   \\
ChebNet                 & $1692.16 \pm 35.28$ & $3216.94 \pm 5.62$ & $4793.49 \pm 6.23$  \\
GIN                     & $755.22 \pm 26.73$  & $953.71 \pm 2.38$  & $3737.39 \pm 22.89$ \\
SGC                     & $203.72 \pm 1.31$   & $226.16 \pm 1.39$  & $439.24 \pm 1.73$   \\
GraphSAGE               & $436.14 \pm 9.02$   & $608.13 \pm 1.18$  & $1240.90 \pm 3.17$  \\\hline
\textbf{GCat} & $155.07 \pm 18.45$  & $458.83 \pm 33.17$ & $821.76 \pm 43.52$ \\\hline
\end{tabular}}
\caption{\label{tab:3}Mean and standard deviation of training time (ms) for all methods on all datasets.}
\end{table}

The prediction performance and the time efficiency of the graph concatenation model can be further improved with either a better classification model fitted on the concatenated matrix or better techniques to generate a graph embedding that can better expose the information of the graph structure. As for now, we can see that the current model adopted in the real-world experiment has already shown competitive performance on prediction as well as a significantly shorter training time comparing to the graph convolution methods.
% \vspace{-10pt}
\section{Conclusion}
In this paper, we studied graph convolution from the perspective of feature selection, quantifying linear and non-linear dependency of the input and output in the node classification task. We first investigate their linear separability with the Fisher score (F-test). Considering every situation can be categorized into homophily and heterophily, their Fisher scores are assessed, and theoretical analysis shows that graph concatenation has a constant advantage over graph convolution. Further, the non-linear dependency is analyzed with mutual information. Similarly, all possible overlap cases are explored, and our inference explicitly confirms the superiority of graph concatenation. To understand their behavior pattern, synthetic data experiments are conducted, where we perform the calculations of the Fisher score and mutual information. The results are consistent with our theoretical study. The corresponding classification is conducted to reveal its possible relationship with Fisher score or mutual information. In experiments on real-world data, the results demonstrate that graph concatenation is comparable with the state-of-the-art GConv-based methods, but with higher efficiency.
{\small
\bibliography{z}
\bibliographystyle{plainnat}}

\newpage

\section*{Appendix}

\subsection*{Proof for Theorem \ref{thm:1}}

\begin{proof}
Treating graph convolution and graph concatenation as pre-processing in the Fisher's linear discriminant~\cite{fisher1936use}, the better transform should make nodes representations of different classes (i.e., $\X_{heter}$, $\X_{homo}$) more distinguishable. 
% An abrupt signal can be defined as the node's feature that is much larger or smaller than the mean of its neighbors and then filtered with a user-defined threshold. 
The effectiveness of a transform can be evaluated by linear separability with the Fisher score, which is defined as:
\begin{equation}
J = J\left(\X_{heter}, \X_{homo}\right)=\frac{\left({\mu}_{heter}-{\mu}_{homo}\right)^{2}}{{\sigma}_{heter}+{\sigma}_{homo}},
\label{eq:ori_fisher}
\end{equation}where $\mu$ and $\sigma$ are mean and variance. $J$ in Eqn. \ref{eq:ori_fisher} can be treated as linear separability on raw data without any transform.
The value of heterogeneous neighborhood is significantly smaller than homogeneous neighborhood, i.e., $\mu_{heter}\ll \mu_{homo}$ or $\mu_{heter}\gg \mu_{homo}$. By Def. \ref{def:lable_homo_heter}
, the variance of heterogeneous neighborhood is much smaller than homogeneous neighborhood, i.e., $\sigma_{homo} \ll \sigma_{heter}$.
Formally, Fisher score of graph convolution can be written as:
\begin{equation*}\footnotesize
J^{GConv}=J\left(\X_{heter}^{\gcat}, \X_{homo}^{\gcat}\right)=\frac{\left({\mu}_{heter}^{\gcat}-{\mu}_{homo}^{\gcat}\right)^{2}}{\sigma_{heter}^{\gcat}+\sigma_{homo}^{\gcat}},
\end{equation*}where superscript $(\cdot)^{GConv}$ means data after applying graph convolution. Graph convolution equals to average one node and its neighbors~\cite{li2018deeper}, so the mean values of heterogeneous and homogeneous neighborhood move to each other, and their different become smaller after applying graph concatenation, i.e., $\left({\mu}_{heter}^{GConv}-{\mu}_{homo}^{GConv}\right)^{2} < \left({\mu}_{heter}-{\mu}_{homo}\right)^{2}$. On average, part of each group move closer to each other, so scatter range of both of two group increases and thereby their variances increase, i.e., ${\sigma}_{heter}< \sigma_{heter}^{GConv}$, and $ {\sigma}_{homo} <\sigma_{homo}^{GConv}$. Therefore, $\sigma_{heter}^{GConv}+\sigma_{heter}^{GConv} > {\sigma}_{heter}+{\sigma}_{homo}$, and we have:
\begin{equation}\small
    J^{GConv}= \frac{\left({\mu}_{heter}^{GConv}-{\mu}_{homo}^{GConv}\right)^{2}}{\sigma_{heter}^{GConv}+\sigma_{homo}^{GConv}}<\frac{\left({\mu}_{heter}-{\mu}_{homo}\right)^{2}}{{\sigma}_{heter}+{\sigma}_{homo}}= J
    \label{eq:j_gc}
\end{equation}
Similarly, graph concatenation is evaluated with Fisher score:

\textbf{(1) For homogeneous neighborhood}, applying graph concatenation on homogeneous neighborhood dose not change any much of either $\A$ and $\X$ since $\X$ is homogeneous regarding the graph or $\A$.
Therefore, the mean and variance of homogeneous neighborhood does not change:  $\small \mu_{homo}^{GCat}= \mu_{homo}$, $\sigma_{homo}^{GCat}= \sigma_{homo}$.

\textbf{(2) For heterogeneous neighborhood}, graph concatenation provides information from $\A$ and $\X$ by their concatenation, which is markedly flexible than the original feature $\X$. The downstream algorithm can tune the weights between $\X$ and $\A$, using $\X$ only to fit the label. Therefore, $J^{\gcat}= J$, since both calculate the Fisher scores on $\X$. Combining Eqn.\ref{eq:j_gc}, we have:
{\small
\begin{equation*}
    J^{\gcat} = J > J^{GConv},
\end{equation*}}which means node representations filtered by graph concatenation is more linear separable than that by graph convolution.
\end{proof}

% \subsection*{Proof for Theorem \ref{thm:2}}

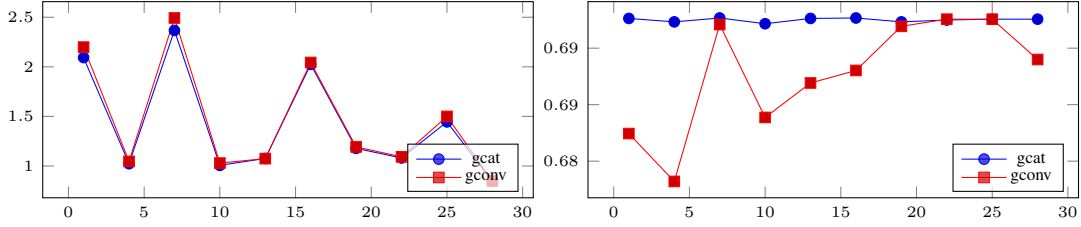
\begin{figure}
\centering
\subfigure{
\begin{tikzpicture}
\begin{axis}[
legend pos=south east,
width=0.58\linewidth,
height=0.3\linewidth,
legend style={nodes={scale=0.65, transform shape},fill=white, fill opacity=0.7, draw opacity=1,text opacity=1}
]
\addplot table {data/case_1_gcat_fs_max.dat};
\addplot table {data/case_1_gconv_fs.dat};
\legend{gcat, gconv}
\end{axis}
\end{tikzpicture}}
\hfill
\subfigure{
\begin{tikzpicture}
\begin{axis}[
legend pos=south east,
width=0.58\linewidth,
height=0.3\linewidth,
legend style={nodes={scale=0.65, transform shape},fill=white, fill opacity=0.7, draw opacity=1,text opacity=1}
]
\addplot table {data/case_1_gcat_mi_max.dat};
\addplot table {data/case_1_gconv_mi.dat};
\legend{gcat, gconv}
\end{axis}
\end{tikzpicture}}
\caption{Case 1: \textbf{(Left)} Fisher scores of graph concatenation and graph convolution, x-axis is eigenvector index. \textbf{(Right)} Mutual information of graph concatenation and graph convolution, x-axis is eigenvector index. }
\label{fig:case_1_mi}
\end{figure}

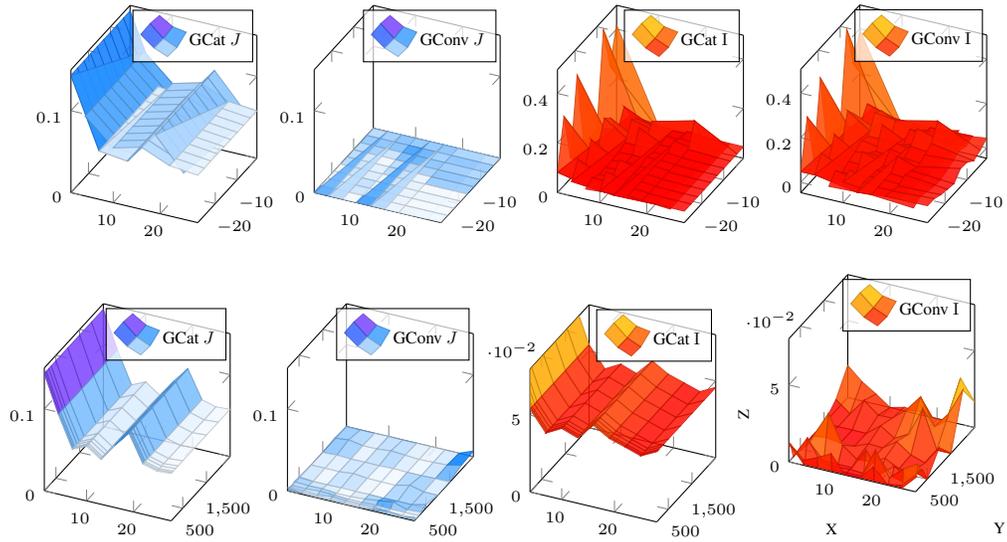
\begin{figure}[!hbpt]
\centering
\subfigure { 
\begin{tikzpicture}
\begin{axis}[
width=0.29\linewidth,
height=0.32\linewidth,
zmin=0, zmax=0.15,
% xlabel =x,
% ylabel =y,
% zlabel =z,
% title = {A Scatter Plot Example}
% label style={font=\tiny},
colormap/cool,
%   colorbar,
legend style={nodes={scale=0.65, transform shape},fill=white, fill opacity=0.7, draw opacity=1,text opacity=1}
]
% this yields a 3x4 matrix in colum
\addplot3 [surf,mesh/rows=20,mesh/ordering=y varies,opacity=0.85] file {data/case_3_0_gcat_fs_max.dat};
\addlegendentry{GCat $J$};
% \addplot3 [only marks]  file {data_3d_grid.dat};
\end{axis}
\end{tikzpicture}}
\subfigure{
\begin{tikzpicture}
\begin{axis}[
width=0.29\linewidth,
height=0.32\linewidth,
zmin=0, zmax=0.15,
% xlabel = $x$,
% ylabel = $y$,
% zlabel = {$f(x,y) = x \cdot y$},
% title = {A Scatter Plot Example}
colormap/cool,
%   colorbar,
legend style={nodes={scale=0.65, transform shape},fill=white, fill opacity=0.7, draw opacity=1,text opacity=1}
]
% this yields a 3x4 matrix in colum
\addplot3 [surf,mesh/check=false,mesh/rows=30,mesh/ordering=y varies,opacity=0.85] file {data/case_3_0_gconv_fs.dat};
\addlegendentry{GConv $J$};
% \addplot3 [only marks]  file {data_3d_grid.dat};
\end{axis}
\end{tikzpicture}}
\subfigure{ 
\begin{tikzpicture}
\begin{axis}[
width=0.29\linewidth,
height=0.32\linewidth,
% zmin=0, zmax=3.5e-2,
% xlabel = $x$,
% ylabel = $y$,
% zlabel = {$f(x,y) = x \cdot y$},
% title = {A Scatter Plot Example}
colormap/redyellow,
%   colorbar,
legend style={nodes={scale=0.65, transform shape},fill=white, fill opacity=0.7, draw opacity=1,text opacity=1}
]
% this yields a 3x4 matrix in colum
\addplot3 [surf,mesh/check=false,mesh/rows=30,mesh/ordering=y varies,opacity=0.85] file {data/case_3_0_gcat_mi_max.dat};
\addlegendentry{GCat $\I$};
% \addplot3 [only marks]  file {data_3d_grid.dat};
\end{axis}
\end{tikzpicture}}
\subfigure{
\begin{tikzpicture}
\begin{axis}[
width=0.29\linewidth,
height=0.32\linewidth,
% zmin=0, zmax=3.5e-2,
% xlabel = $x$,
% ylabel = $y$,
% zlabel = {$f(x,y) = x \cdot y$},
% title = {A Scatter Plot Example}
colormap/redyellow,
%   colorbar,
legend style={nodes={scale=0.65, transform shape},fill=white, fill opacity=0.7, draw opacity=1,text opacity=1}
]
% this yields a 3x4 matrix in colum
\addplot3 [surf,mesh/check=false,mesh/rows=30,mesh/ordering=y varies,opacity=0.85] file {data/case_3_0_gconv_mi.dat};
\addlegendentry{GConv $\I$};
% \addplot3 [only marks]  file {data_3d_grid.dat};
\end{axis}
\end{tikzpicture}}
\subfigure { 
\begin{tikzpicture}
\begin{axis}[
width=0.29\linewidth,
height=0.32\linewidth,
zmin=0, zmax=0.15,
ytick={500,1500},
% xlabel = $x$,
% ylabel = $y$,
% zlabel = {$f(x,y) = x \cdot y$},
% title = {A Scatter Plot Example}
colormap/cool,
%   colorbar,
legend style={nodes={scale=0.65, transform shape},fill=white, fill opacity=0.7, draw opacity=1,text opacity=1}
]
% this yields a 3x4 matrix in colum
\addplot3 [surf,mesh/check=false,mesh/rows=10,mesh/ordering=y varies, opacity=0.85] file {data/case_3_1_gcat_fs_max.dat};
\addlegendentry{GCat $J$};
% \addplot3 [only marks]  file {data_3d_grid.dat};
\end{axis}
\end{tikzpicture}}
\subfigure{
\begin{tikzpicture}
\begin{axis}[
width=0.29\linewidth,
height=0.32\linewidth,
zmin=0, zmax=0.15,
ytick={500,1500},
% xlabel = $x$,
% ylabel = $y$,
% zlabel = {$f(x,y) = x \cdot y$},
% title = {A Scatter Plot Example}
colormap/cool,
%   colorbar,
legend style={nodes={scale=0.65, transform shape},fill=white, fill opacity=0.7, draw opacity=1,text opacity=1}
]
% this yields a 3x4 matrix in colum
\addplot3 [surf,mesh/check=false,mesh/rows=10,mesh/ordering=y varies, opacity=0.85] file {data/case_3_1_gconv_fs.dat};
\addlegendentry{GConv $J$};
% \addplot3 [only marks]  file {data_3d_grid.dat};
\end{axis}
\end{tikzpicture}}
\subfigure{
\begin{tikzpicture}
\begin{axis}[
width=0.29\linewidth,
height=0.32\linewidth,
zmin=0, zmax=8e-2,
ytick={500,1500},
% xlabel = $x$,
% ylabel = $y$,
% zlabel = {$f(x,y) = x \cdot y$},
% title = {A Scatter Plot Example}
colormap/redyellow,
%   colorbar,
legend style={nodes={scale=0.65, transform shape},fill=white, fill opacity=0.7, draw opacity=1,text opacity=1}
]
% this yields a 3x4 matrix in colum
\addplot3 [surf,mesh/check=false,mesh/rows=10,mesh/ordering=y varies, opacity=0.85] file {data/case_3_1_gcat_mi_max.dat};
\addlegendentry{GCat $\I$};
% \addplot3 [only marks]  file {data_3d_grid.dat};
\end{axis}
\end{tikzpicture}}
\subfigure{
\begin{tikzpicture}
\begin{axis}[
width=0.29\linewidth,
height=0.32\linewidth,
zmin=0, zmax=8e-2,
ytick={500,1500},
xlabel = X,
ylabel = Y,
zlabel = Z,
label style={font=\tiny},
colormap/redyellow,
%   colorbar,
legend style={nodes={scale=0.65, transform shape},fill=white, fill opacity=0.7, draw opacity=1,text opacity=1}
]
% this yields a 3x4 matrix in colum
\addplot3 [surf,mesh/check=false,mesh/rows=10,mesh/ordering=y varies, opacity=0.85] file {data/case_3_1_gconv_mi.dat};
\addlegendentry{GConv $\I$};
% \addplot3 [only marks]  file {data_3d_grid.dat};
\end{axis}
\end{tikzpicture}}
\caption{Case 3: \textbf{x-axis}: $\Y$. \textbf{y-axis}: \textbf{[Top Four (option 1)]} $\X$ is set to be eigenvectors (y-axis), or \textbf{[Bottom Four (option 2)]} $\X$ is set to point numbers. \textbf{z-axis} indicates the values of $J$ or $\I$. (all axis are marked in the last sugfigure)}
\label{fig:case_3}
\end{figure}

\begin{figure}
    \centering
    \includegraphics[width=0.7\linewidth]{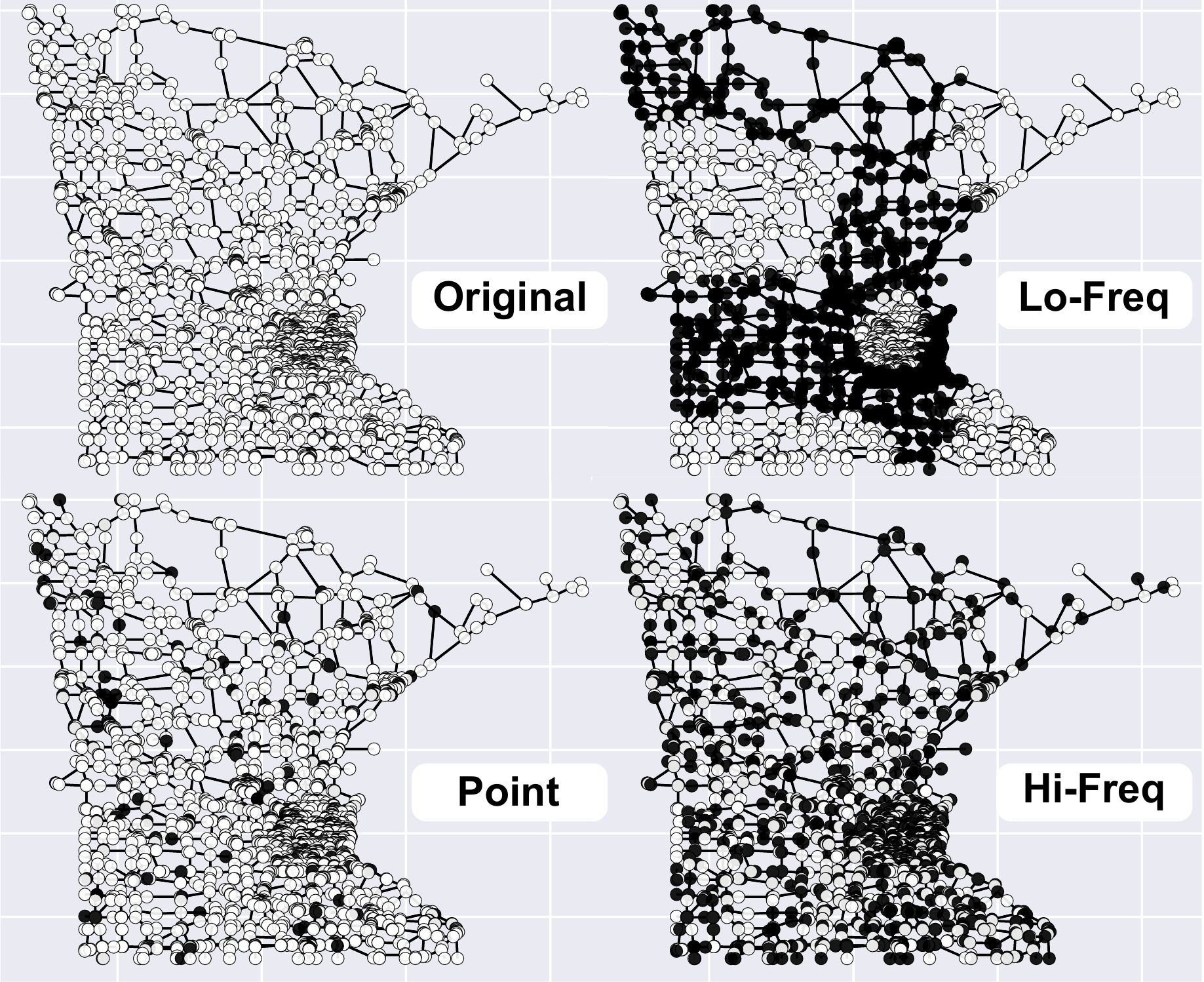}
    \caption{Synthetic data generation: \textbf{(Original)} the Minnesota road network; \textbf{(Lo-Freq)} Low frequency component of the graph Laplacian as homophily label;  \textbf{(Hi-Freq)} High frequency component of the graph Laplacian as heterophily label; \textbf{(Point)} Scattering point as heterophily label;}
    \label{fig:syn_data}
\end{figure}

\end{document}